\documentclass{article}

\usepackage{microtype}
\usepackage{graphicx}
\usepackage{subcaption}
\usepackage{booktabs} 
\usepackage{tcolorbox}

\usepackage{hyperref}

\usepackage[preprint]{icml2026}

\usepackage{amsmath}
\usepackage{amssymb}
\usepackage{mathtools}
\usepackage{amsthm}
\usepackage{xspace}

\usepackage[capitalize,noabbrev]{cleveref}

\theoremstyle{plain}
\newtheorem{theorem}{Theorem}[section]
\newtheorem{proposition}[theorem]{Proposition}

\theoremstyle{definition}

\newtheorem{assumption}[theorem]{Assumption}
\theoremstyle{remark}

\usepackage[textsize=tiny]{todonotes}

\newcommand{\Algo}[1]{\textsc{#1}}

\newcommand{\pbandit}{\Algo{TBDFiltering}\xspace}
\newcommand{\TBDF}{TBDF}
\icmltitlerunning{Submission and Formatting Instructions for ICML 2026}
    
\begin{document}

\twocolumn[
  \icmltitle{\pbandit: Sample-Efficient Tree-Based Data Filtering}



  \icmlsetsymbol{equal}{*}

  \begin{icmlauthorlist}
    \icmlauthor{Robert Istvan Busa-Fekete}{yyy}
    \icmlauthor{Julian Zimmert}{comp}
    \icmlauthor{Anne Xiangyi Zheng}{zzz}
    \icmlauthor{Claudio Gentile}{sch}
    \icmlauthor{Andr\`as Gyorgy}{ttt}
  \end{icmlauthorlist}

  \icmlaffiliation{yyy}{Google Research, NY, USA}
  \icmlaffiliation{comp}{Google Research, Berlin, Germany}
  \icmlaffiliation{zzz}{Google Research, Mountain View, CA, USA}
  \icmlaffiliation{sch}{Google Research, NY, USA}
\icmlaffiliation{ttt}{Google DeepMind, London, UK}

  \icmlcorrespondingauthor{Robert Istvan Busa-Fekete}{busarobi@google.com}
  \icmlcorrespondingauthor{Julian Zimmert}{zimmert@google.com}
  \icmlcorrespondingauthor{Anne Xiangyi Zheng}{annezheng@google.com}
  \icmlcorrespondingauthor{Claudio Gentile}{cgentile@google.com}
  \icmlcorrespondingauthor{Andr\`as Gyorgy}{agyorgy@deepmind.com}

  \icmlkeywords{Machine Learning, ICML}

  \vskip 0.3in
]



\printAffiliationsAndNotice{}  

\begin{abstract}
The quality of machine learning models depends heavily on their training data. Selecting high-quality, diverse training sets for large language models (LLMs) is a difficult task, due to the lack of cheap and reliable quality metrics. While querying existing LLMs for document quality is common, this is not scalable to the large number (billions) of documents used in training. Instead, practitioners often use classifiers trained on sparse quality signals. In this paper, we propose a text-embedding-based hierarchical clustering approach that adaptively selects the documents to be evaluated by the LLM to estimate cluster quality. We prove that our method is query efficient: under the assumption that the hierarchical clustering contains a subtree such that each leaf cluster in the tree is pure enough (i.e., it mostly contains either only good or only bad documents), with high probability, the method can correctly predict the quality of each document after querying a small number of documents. The number of such documents is proportional to the size of the smallest subtree with (almost) pure leaves, without the algorithm knowing this subtree in advance. Furthermore, in a comprehensive experimental study, we demonstrate the benefits of our algorithm compared to other classifier-based filtering methods.  
\end{abstract}

\section{Introduction}
The quality of machine learning models depends heavily on their training data. While traditional machine learning problems usually come with a natural training set, this is not the case for large language models (LLMs), and foundation models in general.
The main and most easily accessible source of data is the World Wide Web, providing petabytes of data. Such data can be readily used for training, but is also inherently diverse in terms of topics and quality. Consequently, data curation and filtering is an essential part of the training of LLMs, including de-duplication to avoid overfitting, topics tagging to ensure diversity, and quality filtering to ensure accuracy \cite{albalak2024a,wang24dataSelectionTutorial}. The main problem in these processes is scalability, due to the sheer volume of data to be processed and the lack of cheap and reliable quality metrics.
In this paper, we focus on scalable methods for quality filtering. 

Historically, data filtering was performed using scalable, easy-to-implement heuristics. Early datasets like C4~\cite{raffel20exploring} applied rule-based filters, such as language identification (e.g., fastText~\cite{joulin16bag}), removal of lines without terminal punctuation, and $n$-gram-based deduplication. More recent large-scale efforts, including FineWeb~\cite{penedo24fineweb} and Dolma~\cite{soldaini24dolma}, employ extensive multi-stage pipelines that utilize URL filtering, quality heuristics (e.g., adapted from Gopher~\cite{rae21gopher}), and content filtering for specific data, such as boilerplate text, toxic language, or personally identifiable information. While effective, these heuristic methods are rigid, can only apply to specific cases, and hence require continuous manual tuning and updates.

To address the limitations of fixed, static heuristics, learned data selection methods have been developed. Various techniques have been proposed to quantify the contribution of data points to model performance, including influence functions \cite{hampel1974influence, koh2017understanding}, or Shapley values \cite{ghorbani2019data, wang2025data}, or bi-level optimization and meta-learning-based methods such as \citep{maclaurin15gradient,pmlr-v48-pedregosa16, lorraine2020optimizing,wang20diffRewards,shen25seal,calian2025datarater}. and perplexity-based signals from pretrained models~\cite{wenzek20ccnet}. However, many of these valuation methods are computationally prohibitive at the scale of billions of documents.

Recently, a more involved data filtering approach has emerged which consists of evaluating an LLM on the content to be selected; these techniques range from extracting such simple metrics as the perplexity assigned by a pretrained model \cite{wenzek20ccnet} to querying an LLM, called \emph{prompting} (e.g., \cite{penedo24fineweb}). Prompts are typically constructed to extract some quality information about the content, such as ``is the content of educational type?'' Though this approach has been found to be effective, prompting LLMs at scale is expensive. To overcome this limitation, the prompting feedback can be generalized through some machine learning approach, where smaller models with fewer parameters are trained to predict the response of larger models used to provide reliable quality evaluations; this latter approach is often referred to as \emph{classifier-based quality filtering} or \emph{student-teacher quality filtering}. 

Our work falls into this line of research, and aims to generalize the prompting feedback of a large and accurate model. However, instead of  training a parametric model (typically a smaller LLM \cite{penedo24fineweb}), our method is {\em non-parametric}, and is based on clusterings built upon suitable embedding models. Since we focus on textual data, we use here text-embeddings \cite{Lee2024GeckoVT}. However, our approach is directly applicable to multi-modal data given the availability of appropriate embedding methods.

At a high level, our quality filtering algorithm works as follows: we cluster the data using the distance in the embedding space,  estimate the quality of the documents in each cluster, and keep the documents in clusters with sufficiently high quality. While this would still be very expensive, due to the potentially (and practically) large number of clusters, we rely on hierarchical clustering (via affinity propagation) to reduce the associated computational cost by utilizing the extra information encoded in the embedding space. Specifically, we successively evaluate the quality of the clusters at a given granularity using sampling, and the split clusters further if their quality cannot be determined. This adaptive approach has great efficiency benefits, as we only evaluate documents using an LLM if their quality cannot be reliably estimated. The resulting algorithm is query efficient: under the assumption that the hierarchical clustering contains a subtree such that each leaf cluster in the tree is approximately pure (i.e., it  mostly contains either only good or bad documents), with high probability, the method can correctly predict the quality of each document after querying a small number of documents. The number of queried documents is proportional to the size of the smallest subtree with (almost) pure leaves, without having the algorithm know this tree in advance. 

We carried out a comprehensive experimental study to assess the impact of our data filtering approach on the downstream 
performance of the pretrained model. More concretely, we trained Gemma 3~\cite{gemmateam2025gemma3technicalreport} models of various sizes from scratch on 172B tokens selected  from three widely used web datasets. As baseline, we sampled training data uniformly at random, and compared this model to the one trained on filtered data by using our proposed filter. We found that the filtering results in a significant improvement on 8 commonly used downstream tasks compared to the performance of the baseline method. In our experiments we used the prompt template from FineWeb Edu filtering~\cite{penedo24fineweb}, another recent classification-based data filtering method. When compared to the latter method, we found that our filtering still results in better data quality. Note, however, that this comparison is still inconclusive, as the two methods used different LLMs to assign quality score to contents.

\section{Problem Formulation}
We are given a dataset $\mathcal{X}$ for which each point $x\in\mathcal{X}$ has a binary quality label $f(x)\in\{0,1\}$. We use $\bar f(\mathcal{X}')=\frac{1}{|\mathcal{X}'|}\sum_{x\in\mathcal{X}'}f(x)$ to denote the average quality of a set $\mathcal{X}' \subset \mathcal{X}$.
The goal is to filter the dataset into two disjoint sets $\mathcal{X}=\mathcal{X}_{\text{keep}}\cup \mathcal{X}_{\text{disc}}$, from which one ($\mathcal{X}_{\text{keep}}$) is kept for training and the other ($\mathcal{X}_{\text{disc}}$) is discarded.
A good filtering satisfies two criteria: a) We want to improve the quality of training data, so $\bar f(\mathcal{X}_{\text{keep}}) \gg \bar f(\mathcal{X})$; and b) we do not want to discard too much useful data, so the quality of the discarded data should be significantly below the original dataset's quality, $\bar f(\mathcal{X}_{\text{disc}}) \ll \bar f(\mathcal{X})$.
Note that these two criteria are not identical when the datasets are unbalanced. We could satisfy the first by providing $\mathcal{X}_{\text{keep}}=\{x_{\text{good}}\}$ with a single good data point. But this would not be a useful filtering captured by the fact that $\bar f(\mathcal{X}_{\text{disc}})\approx \bar f(\mathcal{X})$.

We assume that the evaluation $f(x)$ is costly to perform, which is the case in our application, where the value of $f$ is obtained by prompting an LLM. It is therefore infeasible to evaluate all instances and simply output $\mathcal{X}_{\text{keep}}=\{x\in\mathcal{X}\,|\,f(x)=1\}$.
To make the problem feasible, we assume access to a hierarchical clustering, that is, a tree $\mathcal{T}$ whose leaves are the data points in $\mathcal{X}$.
For a node $n\in\mathcal{T}$, we overload the definition $\bar f(n)=\bar f(D(n))$, where 
\begin{align*}
D(n)=\{x\in\mathcal{X}\,|\,n\text{ is ancestor of $x$ in $\mathcal{T}$}\}~.
\end{align*}
Hence $\bar f(n)$ is the average value of $f(\cdot)$ when restricting to the subtree of $\mathcal{T}$ rooted at $n$.

\begin{figure*}
    \centering
    \includegraphics[width=0.9\textwidth]{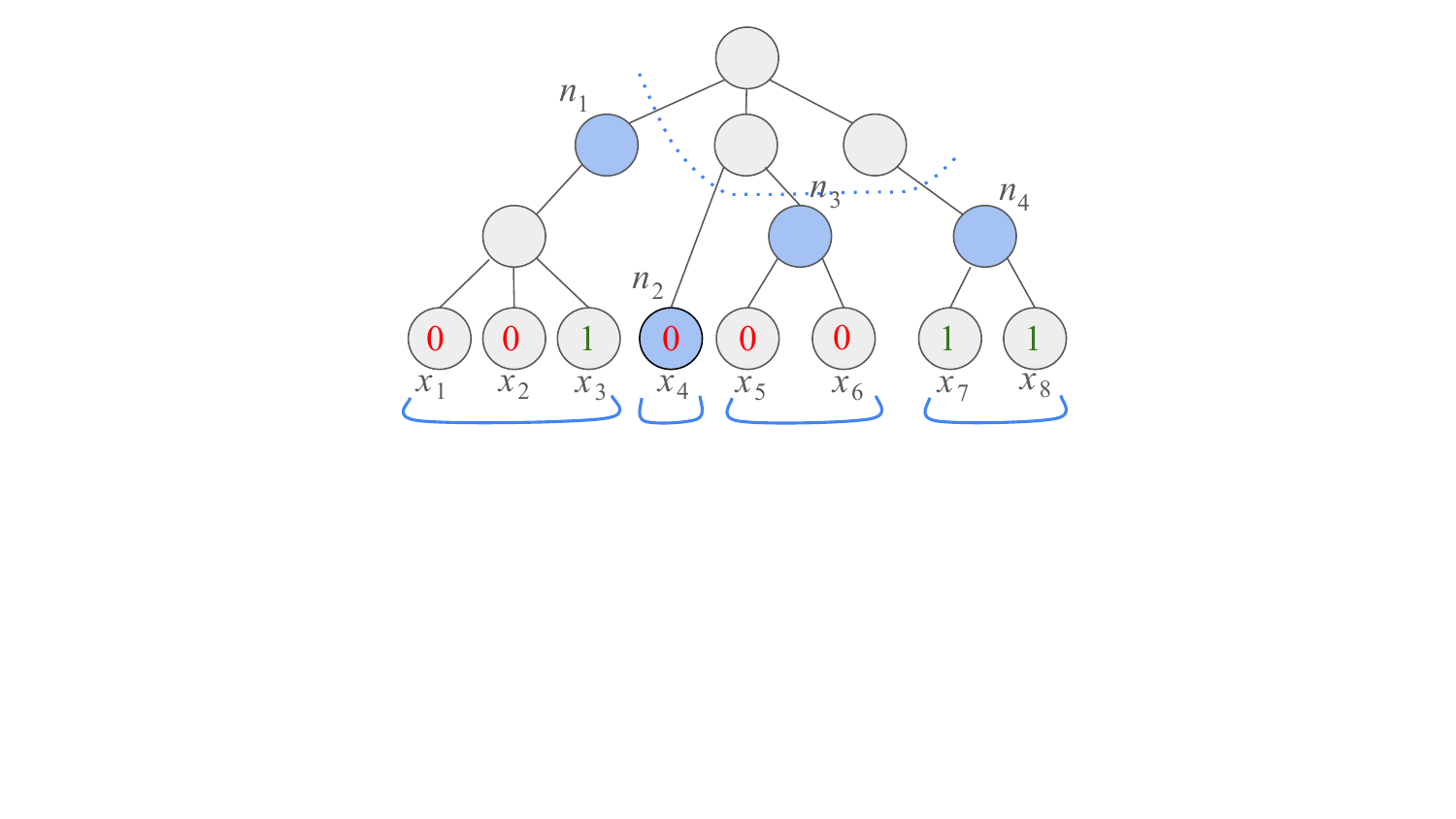}
    \vspace{-1.5in}
    \caption{A hierarchical clustering of the  data items $\mathcal{X} = \{x_1,\ldots, x_8\}$. A cut of the tree is depicted as the dotted line. This is a collection of subtrees whose leaves partition $\mathcal{X}$. In this example, the partition induced by the depicted cut is made up of the four subtrees rooted at the blue nodes $n_1,\ldots,n_4$ corresponding to the clusters $D(n_1) = \{x_1,x_2,x_3\}$, $D(n_2) = \{x_4\}$, $D(n_3) = \{x_5,x_6\}$, and $D(n_4)=\{x_7,x_8\}$. The value of the 0/1 quality label $f(\cdot)$ is also depicted inside the leaves. Assumption~\ref{ass:1} is satisfied with $\alpha' \geq 1/3$, and $\beta'=0$, as the subtrees have average $f$ value of $\bar f(n_1) = 1/3$, $\bar f(n_2) = 0$, $\bar f(n_3) = 0$, and $\bar f(n_4) = 1$.\label{f:1}}
\end{figure*}

A tree $\mathcal{T}$ gives rise to a family of \emph{induced partitions} of $\mathcal{X}$, each of which is identified by a collection of nodes $\{n_1,\ldots,n_K\}$ of $\mathcal{T}$ such that $D(n_1),\ldots,D(n_K)$ form a \emph{partition} of $\mathcal{X}$ (that is, they are disjoint and their union contains all elements of $\mathcal{X}$). We call $n_1,\ldots,n_K$ the nodes of the partition, and its complexity is simply the number $K$ of sets (or nodes) it contains. These are the nodes at the lower border of a {\em cut} of $\mathcal{T}$ (see also Figure~\ref{f:1}).

\begin{assumption}\label{ass:1}
The tree $\mathcal{T}$ contains meaningful structure about the data. Specifically, we assume that there exists a low complexity partition induced by the tree and a pair of nonnegative parameters $\alpha', \beta'$, with $\alpha'+\beta' < 1$, such that for each node $n$ in the partition, it holds $\bar f(n) \leq \alpha'$ or $\bar f(n)\geq 1-\beta'$. 
\end{assumption}

Assumption~\ref{ass:1}  is illustrated in Figure~\ref{f:1}. 
We further assume that every non-leaf node of the tree has at least two children. This is without loss of generality. We can obtain a pruned tree on-the-fly by recursively skipping over nodes that have a single child whenever we collect the children of a node.

\section{Algorithm and Theoretical Analysis}
We propose a greedy algorithm to carry out our filtering, the Tree-Based Data Filtering Algorithm, \pbandit\, for short, illustrated in Algorithm~\ref{alg: greedy}.
The algorithm maintains over time a set of active nodes $\mathcal{A}$, and it starts with marking the root of the tree as active. 
The algorithm evaluates the average quality of an active node $n$ by sampling $N_{\max}$ leaves under node $n$ uniformly at random, evaluate all samples and construct the empirical quality $\widehat f(n)$ as the mean of these evaluations.
If the empirical mean is sufficiently pure, i.e., $\widehat f(n) \geq 1-\beta$ or $\widehat f(n) \leq \alpha$, respectively, we classify all leaves under this node to be either \emph{kept} or \emph{discarded}.
If neither of these conditions hold, we remove the node from $\mathcal{A}$ and instead append all of its children to it.
The algorithm continues until it eventually runs out of nodes to evaluate, which in the worst-case happens when it reaches the bottom of the tree.

\begin{algorithm}
\begin{algorithmic}
\STATE {\bf input:} $\alpha, \beta, N_{\max}$
\STATE $\mathcal{A}\gets \{root(\mathcal{T})\}$, $\mathcal{X}_{\text{disc}}=\{\},\mathcal{X}_{\text{keep}}=\{\}$
\WHILE{$\mathcal{A}$ is not empty}
\STATE Remove $n$ from $\mathcal{A}$. Create estimator $\widehat f(n)$ of $\bar f(n)$ by sampling and evaluating $N_{\max}$ leaves uniformly drawn from $D(n)$.
\IF{$\widehat f(n) \ge 1-\beta$}
\STATE $\mathcal{X}_{\text{keep}}\gets \mathcal{X}_{\text{keep}} \cup D(n)$
\ENDIF
\IF{$\widehat f(n) \le \alpha$}
\STATE $\mathcal{X}_{\text{disc}}\gets \mathcal{X}_{\text{disc}} \cup D(n)$
\ENDIF
\IF{$\widehat f(n) \in (\alpha,1-\beta)$}
\STATE Add all children of $n$ to $\mathcal{A}$
\ENDIF
\ENDWHILE
\end{algorithmic}
\caption{The greedy algorithm \pbandit.}
\label{alg: greedy}
\end{algorithm}

The following result gives sample complexity guarantees for our algorithm, as a function of $K$, the complexity of the final tree cut, defined as the number of subtrees the algorithm classifies (the blue nodes in Figure~\ref{f:1}).

\begin{proposition}
\label{lem:purity}
\Cref{alg: greedy} will produce a filtering $\mathcal{X}_{\text{disc}},\mathcal{X}_{\text{keep}}$ such that for any $\delta \in (0,1)$, with probability at least $1-\delta$, 
\[
\bar f(\mathcal{X}_{\text{disc}}) \leq \alpha + \sqrt{\frac{\log
\big(1.3K/\delta\big)}{N_{\max}}}
\]
and 
\[
\bar f(\mathcal{X}_{\text{keep}}) \geq 1-\beta - \sqrt{\frac{\log
\big(1.3K/\delta\big)}{N_{\max}}}\,,
\]
where $K$ is a random variable denoting the complexity of the final tree cut.
The sample complexity of the algorithm is bounded by $2K N_{\max}$.
\end{proposition}

\begin{proof}
    Let $\{\mathcal{T}_1,\mathcal{T}_2,\dots,\mathcal{T}_K\}$ be the collection of subtrees contained in the final cut produced by the algorithm, and let $\mathcal{N}_{\text{disc}},\mathcal{N}_{\text{keep}}$ denote the roots of the subtrees that are classified as \emph{keep} or \emph{discard} respectively.
    Since $\mathcal{X}_{\text{disc}}$ is exactly the union of the leaves under the nodes of $\mathcal{N}_{\text{disc}}$, we have $\bar f(\mathcal{X}_{\text{disc}}) \leq \max\{\bar f(n)\,|\,n\in\mathcal{N}_{\text{disc}} \}$, and analogously for $\bar f(\mathcal{X}_{\text{keep]}}) \geq \min\{\bar f(n)\,|\,n\in\mathcal{N}_{\text{keep}} \}$.

    By the definition of the algorithm, we further have a bound on the empirical quality of the nodes: $\widehat f(n) \leq \alpha$ for $n\in\mathcal{N}_{\text{disc}}$ and $\widehat f(n) \geq 1-\beta$ for $n\in\mathcal{N}_{\text{keep}}$.
    To prove the lemma, it remains to bound $|\widehat f(n)-\bar f(n)|$ for all the roots of the final cut.
    
    By Hoeffding's inequality, the empirical mean $\widehat\mu$ of $N_{\max}$ i.i.d. samples of a Bernoulli random variable with mean $\mu$ satisfies, with probability at least $1-\delta'$, 
    $$
    |\mu-\widehat\mu|\leq \sqrt{\frac{\log(2/\delta')}{2N_{\max}}}~.
    $$
    We cannot directly apply Hoeffding's inequality to the nodes in $\mathcal{N}_{\text{disc}},\mathcal{N}_{\text{keep}}$ because the identity of these elements is random based on the evaluation outcome of $\widehat f$. Instead we note that the evaluation of the $i$-th node by the algorithm is conditionally independent, hence using a sequence of $\delta'_1,\delta'_2,\dots$ with $\delta'_i=\frac{6\delta}{i^2\pi^2}$ ensures by union bound ( $\sum_{i=1}^\infty\delta'_i=\delta$), that with probability at least $1-\delta$, the $i$-th estimation of the algorithm satisfies Hoeffding's bound with $\delta'_i$, simultaneously for all $i$. 
    Since each node has at least two children, 
    if the final cut has complexity $K$, the algorithm evaluated no more than $2K$ nodes to terminate. Any node $n\in\mathcal{N}_{\text{disc}},\mathcal{N}_{\text{keep}}$ has been evaluated at time $i\leq 2K$. Hence we have under the good event with probability at least $1-\delta$:
    \begin{align*}
    |\bar f(n)-\widehat f(n)| 
    & \leq 
    \sqrt{\frac{\log(2/\delta'_i)}{2N_{\max}}} \\ 
    & \leq 
    \sqrt{\frac{\log(2/\delta'_{2K})}{2N_{\max}}}\\
    &<  
    \sqrt{\frac{\log(1.3K/\delta)}{N_{\max}}}\,.
    \end{align*}
Combining with the above concludes the proof.    
\end{proof}
The worst case complexity of $K$ is the size of dataset $\mathcal{X}$, and we assumed that evaluating these many data points is infeasible. Hence we need a bound on $K$ to meaningfully apply this algorithm. 
The following proposition guarantees such a bound with high probability under \Cref{ass:1}.
\begin{proposition}
Let $K'$ be the complexity of the tree cut with thresholds $\alpha',\beta'$ given by \Cref{ass:1}. Running \Cref{alg: greedy} with any $\alpha,\beta, N_{\max}$ satisfying
$$
\alpha \geq \alpha' + \sqrt{\frac{\log(1.3K'/\delta)}{N_{\max}}}
$$ 
and
$$
\beta \geq \beta' + \sqrt{\frac{\log(1.3K'/\delta)}{N_{\max}}}
$$ 
guarantees with probability at least $1-\delta$ that the complexity of the final tree cut is bounded by $K\leq K'$.
\end{proposition}
\begin{proof}
By the same argument as in the proof of \Cref{lem:purity}, with probability at least $1-\delta$, the first $2K'$ evaluations of nodes by the algorithm satisfy 
$$
|\bar f(n)-\widehat f(n)|\leq  \sqrt{\frac{\log(1.3K'/\delta)}{N_{\max}}}~.
$$
Under this event, any root node of the subtrees in the cut produced by thresholds $\alpha',\beta'$ will not be split further if they are evaluated in the first $2K'$ evaluations of the algorithm.
However, since the total number of nodes above a cut of complexity $K'$ is at most $K'$, the algorithm under this event will either classify all these nodes at a higher level or reach them within the first $2K'$ rounds.
\end{proof}

It is important to note that the preceding analysis does not strictly require the quality signals, $f(x)$, to be binary. The same theoretical framework applies verbatim to continuous signals where, e.g., $f(x) \in [0,1]$. Under this generalization, the fundamental components of our approach remain unchanged: First, the concepts of cluster purity and the upper and lower thresholds encoded by $\alpha$ and $\beta$ remain entirely valid. Second, the core goal remains the separation of high-quality documents (where $f(x)$ is near 1) from low-quality ones (where $f(x)$ is near 0). While the experiments detailed in Sections \ref{sec: implementation} and \ref{sec: results} utilize ordinal scores,
rather than binary values, this distinction is rather inessential, and does not detract from the applicability of what is contained in this section.

\section{Implementation Details}
\label{sec: implementation}

The \pbandit algorithm is a generic algorithmic scheme which can be used with various text embeddings, clustering algorithms and prompted LLM model. We summarize our modeling choices in this section, along with the ablation setup we used.

\subsection{Text representation}

We embed each document by using the Gecko embedding~\cite{Lee2024GeckoVT}. We used the \texttt{text-multilingual-embedding-002} model\footnote{\url{https://docs.cloud.google.com/vertex-ai/generative-ai/docs/embeddings/get-text-embeddings}} whose embedding dimension is 768. This model can handle only $2048$ tokens and any excess is silently truncated. Since we would like to evaluate whole documents, each document that is longer than the token limit is chunked and the chunks are embedded separately. If a chunk is too short ($<$ 50 tokens), we neglected the chunk. So the algorithm is run at chunk level, and ultimately the score of a document is derived from the chunk level scores described later by taking the average score of \pbandit over the chunks. Note that more than $95\%$ of the documents consist of less than 2048 tokens, thus the chunking has  a marginal impact overall.

\subsection{Affinity clustering}

We use the distributed hierarchical clustering algorithm proposed by
\citet{bateni2017affinity}.
This algorithm builds on the classical Bor\r{u}vka's algorithm \cite{boruuvka1926jistem} for finding Minimum spanning trees (MST). The algorithm maintains a set of clusters initialized to the family of singleton clusters.
In every round, it computes in parallel the shortest edge of each cluster to another cluster, which in our setting is based on the cosine similarity of the Gecko embedding.
Finally, it merges all clusters that are connected by shortest edges computed in this round.
This algorithm can be stopped at any prespecified number of clusters since it produces a hierarchical order of clusters.

Our empirical observations indicate that this algorithm generates well-balanced clusters. This stability stems from the merging process: because every cluster is paired with at least one other in each iteration, the minimum cluster size at least doubles per round. By applying five compression rounds, the algorithm produces a balanced hierarchical tree with a total depth of six.

\subsection{Ablation}
\label{sec:ablation}

In our ablation setting we used three Gemma 3 models~\citep{gemmateam2025gemma3technicalreport} with 270M, 1B and 4B parameters, respectively. We trained each model for 172B token horizon with 32,768 sequence length. The batch size was set to 64 for the 270M and 1B, and 256 for the 4B model. We used the same learning rate and optimization process as in Gemma 3. All models are trained from scratch. These baselines are compared to the corresponding models trained with the very same parameters, but where data is filtered via \pbandit by taking the documents with top-$k$ scores so as it consists of training data with 172B tokens, and the comparison is carried out in an iso-flop setup.

We carried out our comparison over eight commonly used evaluation benchmarks:  HellaSwag~\cite{zellers2018swagaf}, WinoGrande~\cite{winogrande}, SIQA~\cite{sap19siqa}, PIQA~\cite{bisk20piqa}, ARC~\cite{clark18arc}, Commonsense QA~\cite{talmor19commonsenseqa}, MMLU~\cite{zhao-et-al-2025-mmlu-cf}, Book QA~\cite{mihaylov2018can}.
Each of these tasks require to compute accuracy based on the response of the model. To decrease the noise in the evaluation, we averaged out the accuracy of the last $10$ checkpoints. The checkpoints are stored every 50 steps. The accuracy scores vary from task to task, therefore we computed the relative improvement for each task, and then averaged them to obtain a result in aggregate form across models. For completeness, we also report the absolute scores in Appendix~\ref{sa:results}.

\subsection{Datasets and statistics}

We evaluated our approach using three standard web-curated benchmarks: FineWeb \cite{penedo2024finewebdatasetsdecantingweb}, ThePile \cite{pile}, and C4 \cite{c4data}. For FineWeb, we implemented URL-based deduplication by grouping entries by URL and randomly selecting a single version for the training set, resulting in a final corpus of 4.13T tokens.

Detailed characteristics of these datasets are provided in Table~\ref{tab:data_sets_main_stat}. Notably, the impact of document chunking was marginal, as typical web document lengths fall well within the context window of Gecko embeddings.

\begin{table}[ht!]
    \centering
    \begin{tabular}{|c|c|c|c|c|}
    \hline
    Dataset & \#Documents & \#Tokens & \#Chunks  \\
    \hline \hline
    C4       & 367M   & 173B  & 394M \\
    ThePile  & 210M   & 334B & 261M \\
    FineWeb  &  5.85B & 4.13T & 5.98T \\
    \hline
    \end{tabular}
    \caption{Main statistics of benchmark datasets.}
    \label{tab:data_sets_main_stat}
\end{table}

\subsection{Prompting}
To assess document quality, we used two different prompts: one is taken from the FineWeb EDU~\cite{penedo2024finewebdatasetsdecantingweb} quality filtering approach, the other  is written by us. The latter prompt, named \Algo{General}, is general purpose, not specifically focusing on educational content. This prompt template can be found in Appendix~\ref{app:prompt_general}. We refer to the \pbandit algorithm as $\Algo{\TBDF(FW-EDU})$ and $\Algo{\TBDF(FW-General})$, depending on which prompt template is used.  Gemini 2.5 Flash\footnote{\url{https://docs.cloud.google.com/vertex-ai/generative-ai/docs/models/gemini/2-5-flash}} is used to prompt in both \pbandit instances. Both prompt templates expect the model to provide an ordinal feedback from 0 to 5. We used this ordinal feedback as a real value, averaged it out and divided by 5 at each cluster. This results in normalized scores into $[0,1]$ that are more aligned with the assumptions of our algorithm. We computed the confidence interval for this score in a Bayesian setting by sampling the posterior distribution with a population of size 100. This sampled population from the posterior is used to compute confidence intervals. In each case, we 
used the non-informative prior, which is the Dirichlet distribution with parameters $(1,1,1,1,1,1)$ in this case. The reason why we used these Bayesian intervals instead of the non-parametric intervals contained in the description and analysis of Algorithm \ref{alg: greedy} is that a credible interval is less conservative and, in practice, it still achieves a good coverage.

\begin{figure}[ht!]
    \centering
    \begin{subfigure}[t]{0.5\textwidth}
    \includegraphics[width=0.95\linewidth]{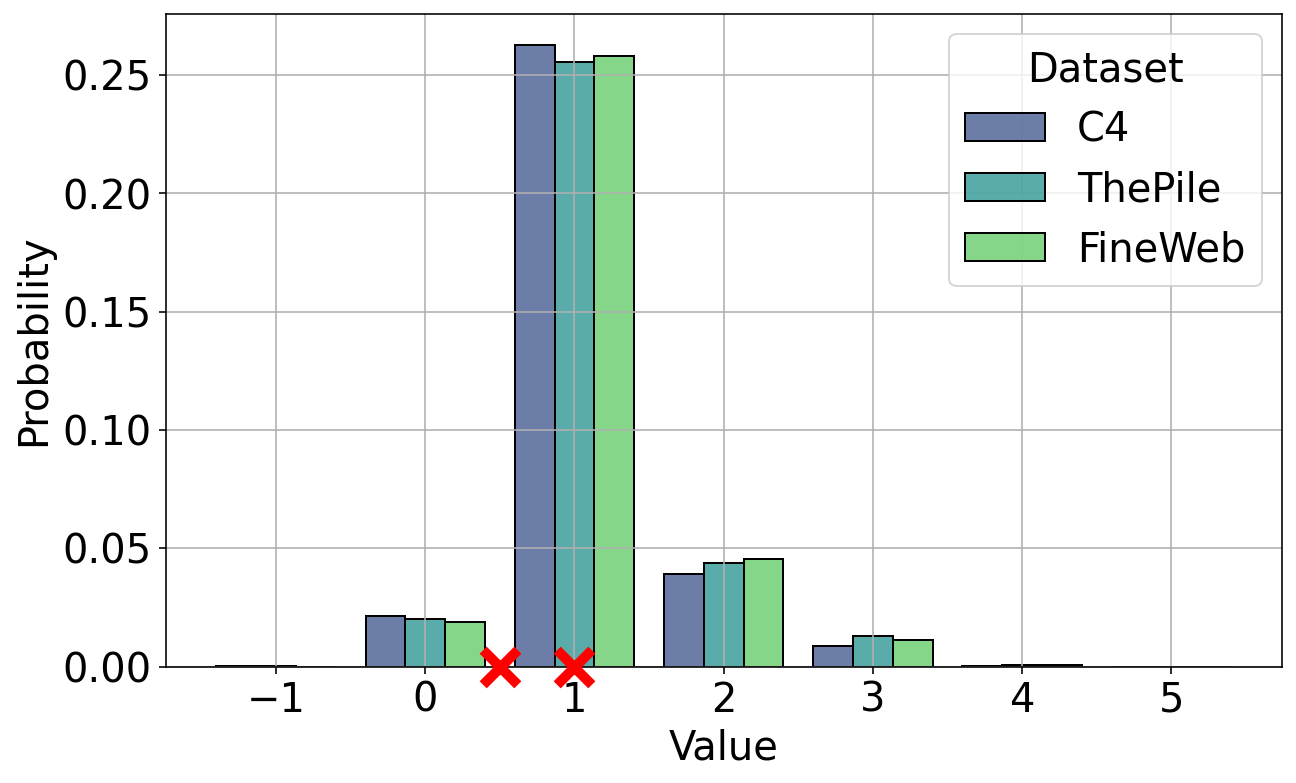}
    \caption{Using \Algo{FW-EDU} prompt}
    \label{fig:finweb_prompt_dist}
    \end{subfigure}
    ~
    \centering
    \begin{subfigure}[t]{0.5\textwidth}
    \includegraphics[width=0.95\linewidth]{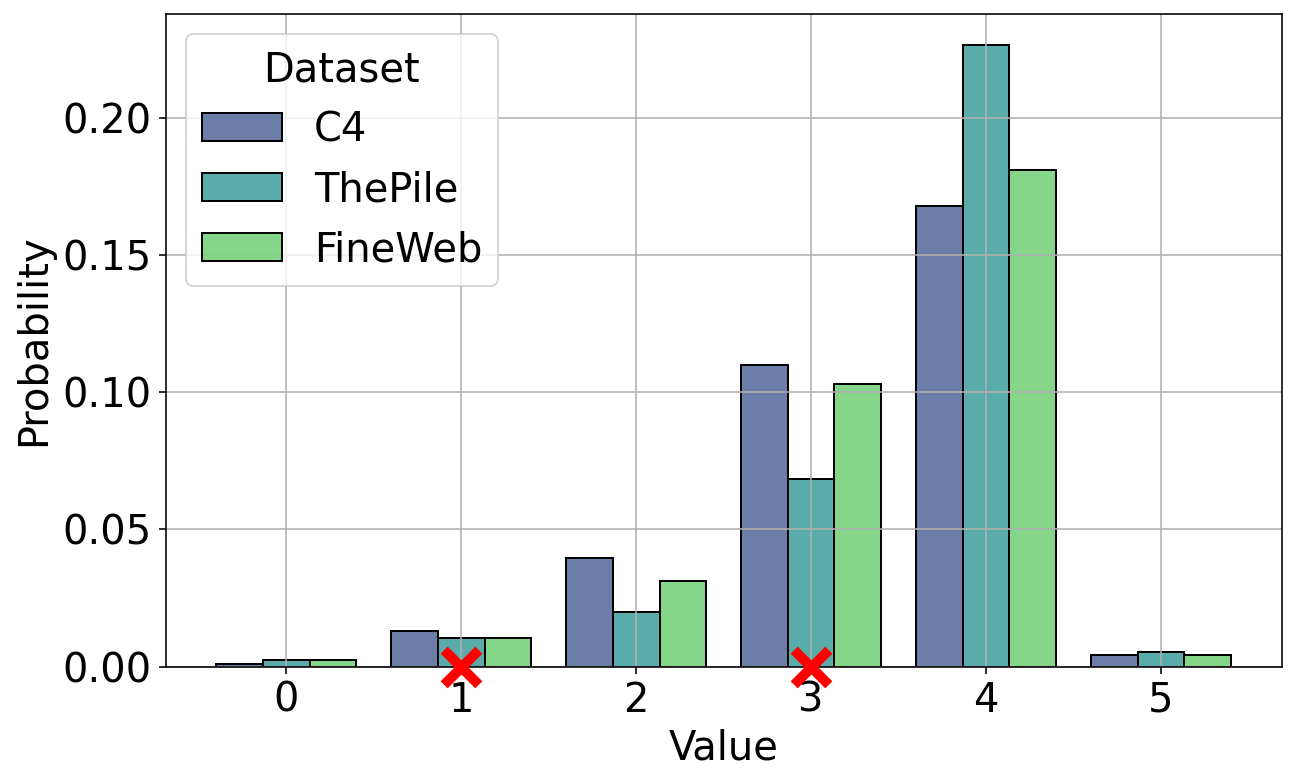}
    \caption{Using \Algo{General} prompt}
    \label{fig:c4}
    \end{subfigure}
    \caption{Prompt feedback distribution computed based on 100k randomly selected chunks from the three datasets and two prompts we used with \pbandit. We indicated the failure of the prompt response by $-1$ which happened very rarely. In our implementation, we assigned to these cases 0 score. The two red crosses in the two histograms correspond to the values of thresholds $\alpha$ and $1-\beta$ that have been selected in the two cases, multiplied by 5.}
    \label{fig:prompt_distribution}
\end{figure}

We selected 100k chunks from each dataset and evaluated them using the \Algo{FW-EDU} and \Algo{General} prompts. The distribution of the ordinal feedback provided by the model can be seen in Figure \ref{fig:prompt_distribution}. As expected, the model response are quite different for the two prompt templates we used, since they implemented quite different evaluation criteria. From practical point of view, these empirical histograms help us pick a feasible pair of threshold parameters $\alpha$ and $1-\beta$ for the algorithm. With the \Algo{FW-EDU} prompt, we set $\alpha=0.1$ and $1-\beta=0.2$ (which in Figure \ref{fig:prompt_distribution} become 0.5 and 1, respectively). With the \Algo{General} prompt, we set $\alpha=0.2$ and $1-\beta=0.6$ (corresponding to 1 and 3 in Figure \ref{fig:prompt_distribution}).

The parameter $N_{\max}$ in Algorithm \ref{alg: greedy} controls the maximum number of prompts per cluster. If the number of chunks that are evaluated exceeds $N_{\max}$, the \pbandit algorithm proceeds by evaluating the child nodes in the tree. In our experiments, we set $N_{\max} = 100$, which implies that the cluster quality is roughly approximated within a $0.01$ additive error in each case. With this setting, less than 10\% of the chunks are evaluated by \pbandit on both prompt templates.

\section{Experimental Results}\label{sec: results}

We report the performance of the Gemma 3 models with various parameters (270M, 1B, 4B), as compared to the baselines.

\subsection{Relative improvement of model performance}

\begin{figure*}[ht!]
    \centering
    \begin{subfigure}[t]{1.0\textwidth}
    \includegraphics[width=1.0\linewidth]{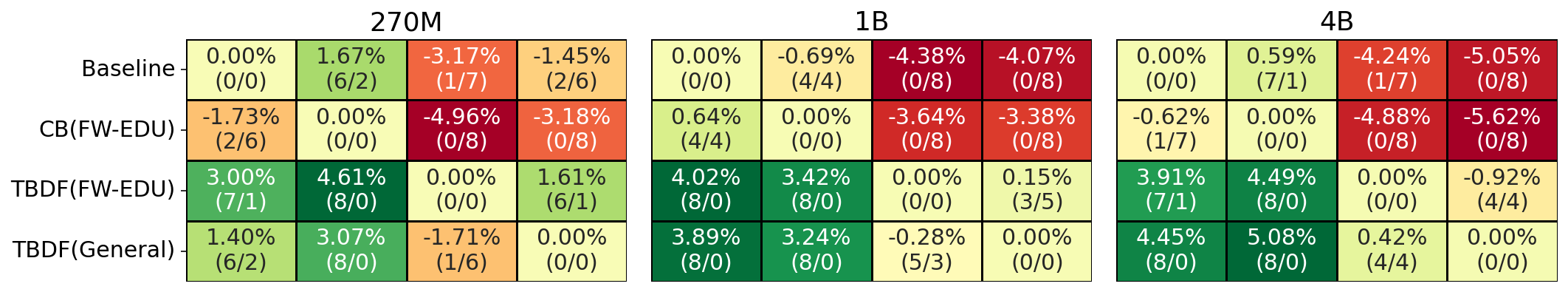}
    \caption{ThePile}
    \label{fig:thepile}
    \end{subfigure}
    ~
    \centering
    \begin{subfigure}[t]{1.0\textwidth}
    \includegraphics[width=1.0\linewidth]{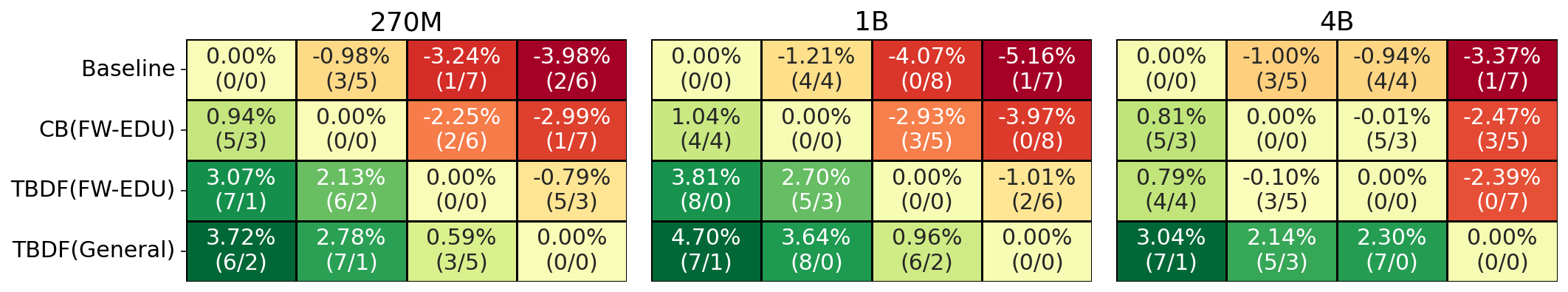}
    \caption{C4}
    \label{fig:c4}
    \end{subfigure}
    ~
    \centering
    \begin{subfigure}[t]{1.0\textwidth}
    \includegraphics[width=1.0\linewidth]{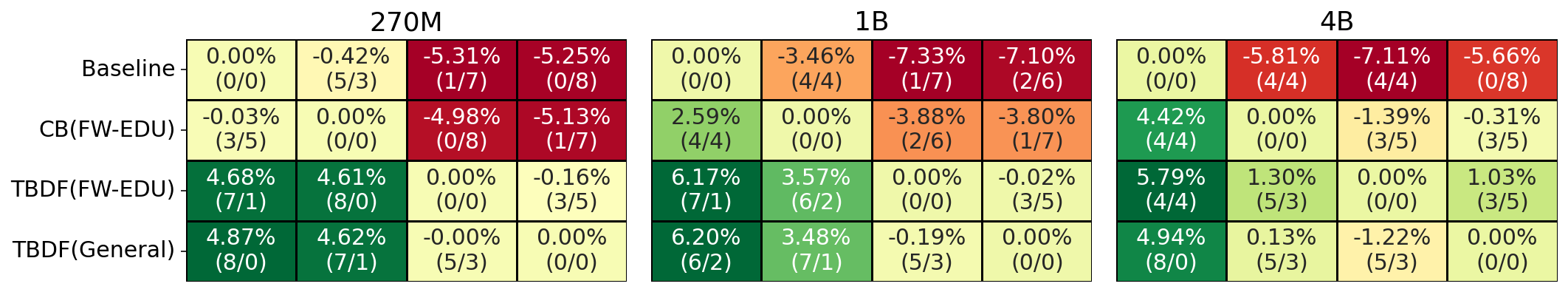}
    \caption{FineWeb}
    \label{fig:fineweb}
    \end{subfigure}
    \caption{The heatmap of average relative improvement comparing different training setups, specifically when no filtering is applied, and when filtering is carried out through \Algo{CB(FW-EDU)}, \Algo{TBDF(FW-EDU)} and \Algo{TBDF(General)}. The results for ThePile, C4 and FineWeb datasets are shown in Figure \ref{fig:thepile}, \ref{fig:c4} and \ref{fig:fineweb}, respectively. In each subfigure, the performance of the Gemma 3 model with parameters 270M, 1B and 4B is shown from left to right.
    Each cell displays the relative improvement score averaged over the 8 evaluation tasks mentioned in Section \ref{sec:ablation}. Numbers within brackets indicate the count of evaluation tasks where the model outperformed versus underperformed against the baseline. For example, a result of 6/2 signifies that the model surpassed the baseline on six tasks while performing inferiorly on two. The cells are colored depending on the relative difference in a pairwise comparison: green means that the ablation is better with the corresponding filtering, red indicates inferior performance.
    }
\end{figure*}

The baseline model is training without data filtering. The results for each dataset, that include ThePile, C4 and FineWeb, can be seen in Figure~\ref{fig:thepile}, Figure~\ref{fig:c4} and Figure~\ref{fig:fineweb}, respectively. In each cell, the average relative improvement is reported along with the number of evaluation tasks where the model was superior and inferior, respectively, separated by a slash ($/$) within the brackets. We run the \pbandit algorithm with two instantiations with various prompt templates: (1) We used the FineWeb-Edu prompt which is used to train the FineWeb EDU data filtering model in a student-teacher framework. We refer to this version of \pbandit filtering as \Algo{\TBDF(FW-EDU)}. (2) We also ran the \pbandit algorithm by using a general purpose prompt template which expects the LLM model to give a general assessment of its content. This prompt is reported in Appendix~\ref{app:prompt_general}. We refer to this version of the \pbandit algorithm as \Algo{\TBDF(General)}.

These results reveal some general trends. First of all, the model trained on filtered data improves its performance over the non-filtered model by  1-5\% in a relative sense for all model sizes. Second, the prompt used in \pbandit does not make too much difference in the relative metric, since the relative improvement of \Algo{\TBDF(FW-EDU)} and \Algo{\TBDF(General)} are on par with each other. The third finding is that the classification-based approach using FineWeb-Edu model is inferior for 270M and 1B, and on par on the 4B model with the \pbandit-based filtering with the very same prompt. This model is referred to as $\Algo{CB(FW-EDU)}$ in Figures~\ref{fig:thepile}-\ref{fig:fineweb}.

\subsection{Where does the improvement come from?}

We compared the filtering methods specific to each evaluation task. Table~\ref{tab:eval_tasks} shows the relative improvement achieved by different filterings computed for each downstream evaluation task separately. \pbandit improves on each evaluation task uniformly no matter the prompt template used in the algorithm. The largest relative improvement of \pbandit is achieved on \Algo{ARC} and \Algo{BookQA}. Similarly, the \Algo{CB(FW-EDU)} also achieved the largest improvement on \Algo{ARC}, but on the summarization task \Algo{BookQA} the  improvement was marginal, unlike \pbandit.
At this token horizon, \Algo{CB(FW-EDU)} was not able to improve on a number of evaluation tasks; however this does not contradict the findings in~\cite{penedo2024finewebdatasetsdecantingweb}, since in that paper the authors trained a different model with a different token horizon. 

\begin{table*}[]
    \centering
    \begin{tabular}{|c|c|c|c|c|c|c|c|c|c|c|}
    \hline
    \hline
    & HellaSwag  & ARC  & Winogrande  & PIQA  \\
    \hline
    TBDF(FW-EDU)	 & $2.36 \quad  (7/2)$ & $14.03 \quad  (9/0)$ & $1.43 \quad  (7/2)$ & $1.35 \quad  (9/0)$ \\
    TBDF(General)	 & $3.81 \quad  (9/0)$ & $10.26 \quad  (8/1)$ & $2.79 \quad  (9/0)$ & $1.89 \quad  (8/1)$ \\
    CB(FW-EDU)	 & $0.04 \quad  (5/4)$ & $11.01 \quad  (6/3)$ & $-0.25 \quad  (3/6)$ & $-0.54 \quad  (3/6)$ \\
    \hline
    \hline
     & SIQA  & Commonsense QA  & MMLU  & BookQA  \\
    \hline
    TBDF(FW-EDU)	 & $0.86 \quad  (7/2)$ & $3.41 \quad  (6/3)$ & $4.46 \quad  (7/2)$ & $7.45 \quad  (7/2)$ \\
    TBDF(General)	 & $1.68 \quad  (8/1)$ & $3.90 \quad  (8/1)$ & $1.52 \quad  (6/3)$ & $10.66 \quad  (8/1)$ \\
    CB(FW-EDU)	 & $-0.56 \quad  (3/6)$ & $-0.09 \quad  (5/4)$ & $-0.61 \quad  (4/5)$ & $1.07 \quad  (3/6)$ \\
            \hline
    \end{tabular}
    \medskip
    \caption{Average relative improvement per evaluation task computed for all models: Gemma 3 with 270M, 1B and 4B parameters. In braces are the counts of the cases, over all datasets and model sizes, when the baseline performance was inferior/superior compared to the model performance trained on filtered data.}
    \label{tab:eval_tasks}
\end{table*}

\subsection{Why does clustering help?}
Finally, we tried to get a better understanding of the extent to which clustering helps in scaling up the prompting. We checked how homogeneous the clusters are with respect to the prompting feedback. 
In this experiment, we used the general purpose prompt \Algo{General}.

First, we sampled $10k$ documents from each of the datasets uniformly at random, and computed the entropy of the prompting feedback for these documents (which can be viewed as a discrete distribution with 6 possible outcomes, from 0 to 5). These entropy values are indicated as a 
dashed red line in Figure~\ref{fig:cluster_purity}.  

\begin{figure*}[ht!]
    \centering
    \includegraphics[width=1.0\linewidth]{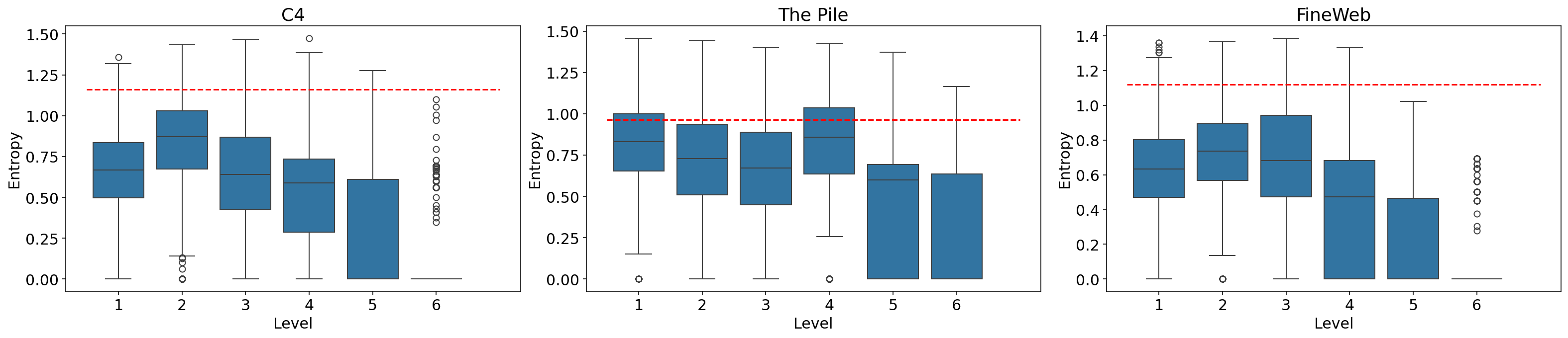}
    \caption{Boxplot representation of the distribution of entropy of model responses at each level of the hierarchy. The entropy tends to decrease as we move from level 1 to level 6.}
    \label{fig:cluster_purity}
\end{figure*}

Then we selected uniformly at random 100 clusters from each level of the cluster hierarchy, prompted 100 instances from each of these clusters, and computed the entropy for each clusters. If the cluster was smaller than $100$, we evaluated all documents from the cluster exhaustively, so the entropy was computed based on fewer prompting feedback than 100. Figure~\ref{fig:cluster_purity} shows the distribution of these one hundred entropy values for each level of the cluster hierarchy, and each dataset separately. One can see that the entropy values tend to be lower as we go deeper into the cluster hierarchy (i.e., from left to right on each subplot), which indicates that the clusters are purer at lower levels, as one would expect. It is also important to note that the entropy of clusters at every level are uniformly lower than if we selected the 10k documents uniformly at random, regardless of the clustering, indicating why clustering helps here.
It is the purity of the clusters at lower levels that provides the ability to 
\pbandit to effectively select high-quality documents. This is in line with what we stated in Assumption~\ref{ass:1}.

\section{Conclusion and Ongoing Work}

In this work, we introduced \pbandit, a sample-efficient, non-parametric approach for filtering large-scale datasets for LLM training. Addressing the critical challenge of scalability in data curation, our method leverages a hierarchical clustering structure to better align with the semantic information contained in the text embeddings. By modeling the filtering process as an adaptive sampling problem on this hierarchy, we significantly reduce the reliance on expensive oracle evaluations (such as LLM prompting) while maintaining high precision in identifying high-quality content.

Our contributions are both theoretical and empirical. Theoretically, we proved that \pbandit is sample efficient, in that its sample complexity is proportional to the logarithm of the structural complexity $K$ of the optimal tree cut rather than the logarithm of the dataset size $n$. This ensures that our algorithm effectively exploits the ``purity'' of clusters--a property we empirically validated by observing decreasing entropy in prompting feedback at deeper levels of the hierarchy.

Empirically, we demonstrated that Gemma 3 models (ranging from 270M to 4B parameters) trained on data selected by \pbandit consistently outperform baselines trained on unfiltered data across diverse benchmarks, including \Algo{HellaSwag}, \Algo{MMLU}, and \Algo{PIQA}. Furthermore, our method remains competitive with, and often superior to, state-of-the-art classifier-based methods like FineWeb-Edu, without the need to train specialized ``student'' scoring models. This performance robustness holds across multiple massive datasets, including ThePile, C4, and FineWeb EDU.

Ongoing work is exploring the application of \pbandit to multi-modal domains, given that the underlying embedding and clustering mechanisms are modality-agnostic. Additionally, further restraints to the filtering process may include suitable notions of {\em diversity} among the retained samples. \pbandit offers a flexible and scalable paradigm for data curation, reducing the computational barrier to training high-quality foundation models.

\section*{Impact Statement}

This paper presents work whose goal is to advance the field of Machine
Learning. There are many potential societal consequences of our work, none
which we feel must be specifically highlighted here.

\bibliographystyle{icml2026}

\newpage
\appendix
\onecolumn

\section{Prompt Template}\label{app:prompt_general}

Below is the general purpose prompt template alluded to in the main text. We refer to this prompt template as \Algo{General} in the main body of the paper.

\begin{tcolorbox}[colframe=black]
        Role: You are a meticulous AI Data Curator. Your primary mission is to evaluate web text to build a diverse and high-quality dataset for training a large language model from scratch. Your goal is to select content that will teach the model language fundamentals, common sense, and a broad understanding of the world.
        Objective: Analyze the provided text and assign it a quality score. The score must reflect the text's utility in teaching a nascent AI about language, facts, reasoning, and concepts. You are building a library for a new mind; it needs everything from children's stories to technical manuals.
        Core Task: You will be given a piece of text content. Evaluate it against the criteria below and provide your assessment in a structured JSON format.

        Core Evaluation Criteria
        1. Linguistic Quality (Weight: 40
          Definition: The text's grammatical correctness, clarity, and structural coherence. This is the most important factor, as the model learns syntax and sentence structure from this.
          High Quality: Sentences are well-formed, grammatically correct, and use clear language. The text is free of significant spelling errors. It is well-structured and easy to read.
          Low Quality: Gibberish, "wall of text," numerous grammatical/spelling errors, or incoherent sentence structure.
        2. Knowledge Value (Weight: 40
          Definition: The informational content of the text. Unlike fine-tuning, both foundational and specialized knowledge are highly valuable.
          High Quality (Two valid types):
          Type A: Foundational \& Common-Sense: Clearly explains basic concepts, facts, or common-sense relationships (e.g., "The sun is a star," "Dogs are mammals that bark," "To make a sandwich, you need bread"). This includes simple narratives, descriptions, and instructions.
          Type B: Specialized \& In-Depth: Provides detailed, factual information on a specific topic (e.g., technical explanations, historical analysis, scientific articles).

          Low Quality: Contains no discernible information, is purely repetitive, or makes claims that are nonsensical.
        3. Safety \& Reliability (Weight: 20
          Definition: The text's trustworthiness and safety. The model should be built on a foundation of reality and avoid learning harmful biases or misinformation.
          High Quality: The content is factually sound (for its context), logical, and coherent. It is free from hate speech, toxicity, and harmful instructions.
          Low Quality: Contains blatant misinformation, conspiracy theories, pseudoscience, hate speech, or dangerous content.

      Nuanced Cases: Advertising \& Product Descriptions
      Valuable: Product descriptions that clearly list features, materials, functions, or benefits are useful. They teach the model about objects, attributes, and descriptive language. (e.g., "This 16-inch laptop features an M3 processor, 16GB of RAM, and a Liquid Retina XDR display.")
      Not Valuable: Purely manipulative or spammy ad copy with no informational content. (e.g., "Click here now! Best deals! Unbelievable prices! Don't wait!")

      Red Flags (Content to be automatically classified as "Very Poor"):
      Immediately score the content as 0 if it is predominantly any of the following:
      Gibberish, garbled text, or placeholder text (e.g., "lorem ipsum").
      Hate speech, extreme toxicity, or directly harmful content.
      Keyword stuffing that renders the text unreadable.
      Purely navigational elements (e.g., a list of links with no context).
      Deceptive or purely manipulative advertising with zero informational value.

      Output Format
      Your response MUST be a JSON object with the following schema:
      \{
        "quality\_score": $<$An integer from 0 to 5$>$,
        "reasoning": "$<$A brief, one-sentence explanation for your score, referencing the criteria.$>$",
        "content\_type": $<$"Foundational", "Specialized", "Descriptive", "Instructional", "Narrative", "Conversational", "Other"$>$,
        "is\_training\_candidate": $<$true or false$>$
      \}
      Here is the document to be rated:    
\end{tcolorbox}

\newpage

\section{Further Results}\label{sa:results}

Table~\ref{tab:ablation} contains absolute scores from our experiments across models, tasks, and filtering methods.

\begin{table}[ht!]
    \centering
    \begin{tabular}{|c|c|c|c|c|c|c|c|c|c|c|c|c|}
\hline \hline Model size & \multicolumn{9}{c|}{270M} \\ \hline
Dataset & Filtering  & HellaSwag	& ARC	& Winogrande	& PIQA	& SIQA	& CQA	& MMLU	& BookQA	\\ \hline \hline 
ThePile & Baseline	& $35.5$	& $25.1$	& $52.9$	& $64.7$	& $44.4$	& $32.4$	& $24.8$	& $25.4$	\\
ThePile  & CB(FW-EDU)	& $34.9$	& $24.7$	& $52.4$	& $65.0$	& $43.6$	& $30.6$	& $25.0$	& $24.6$	\\
ThePile & \TBDF(FW-EDU)	& $36.2$	& $26.1$	& $54.1$	& $66.8$	& $44.6$	& $35.1$	& $26.2$	& $25.2$	\\
ThePile & \TBDF(General)	& $35.6$	& $25.9$	& $54.1$	& $66.3$	& $44.7$	& $32.2$	& $25.7$	& $25.2$	\\
\hline
C4 & Baseline	& $41.7$	& $25.0$	& $51.2$	& $70.0$	& $44.1$	& $32.7$	& $25.4$	& $25.1$	\\
C4  & CB(FW-EDU)	& $42.4$	& $25.0$	& $53.6$	& $69.7$	& $44.9$	& $32.8$	& $25.8$	& $24.8$	\\
C4 & \TBDF(FW-EDU)	& $42.3$	& $25.3$	& $52.6$	& $71.9$	& $46.0$	& $35.1$	& $27.2$	& $25.0$	\\
C4 & \TBDF(General)	& $42.8$	& $24.8$	& $54.9$	& $69.8$	& $45.6$	& $33.9$	& $26.8$	& $27.8$	\\
\hline
FineWeb  & Baseline	& $40.1$	& $25.0$	& $52.7$	& $69.6$	& $44.2$	& $32.5$	& $26.0$	& $25.0$	\\
FineWeb  & CB(FW-EDU)	& $40.3$	& $29.6$	& $51.6$	& $67.0$	& $43.3$	& $32.9$	& $25.1$	& $23.7$	\\
FineWeb  & \TBDF(FW-EDU)	& $42.2$	& $30.3$	& $53.5$	& $69.9$	& $44.3$	& $34.4$	& $25.7$	& $27.3$	\\
FineWeb  & \TBDF(General)	& $43.4$	& $27.6$	& $53.1$	& $72.1$	& $44.8$	& $35.5$	& $26.4$	& $26.6$	\\
\hline
\hline Model size & \multicolumn{9}{c|}{1B} \\ \hline
Dataset & Filtering  & HellaSwag	& ARC	& Winogrande	& PIQA	& SIQA	& CQA	& MMLU	& BookQA	\\ \hline \hline 
ThePile & Baseline	& $51.2$	& $29.4$	& $54.9$	& $71.8$	& $45.3$	& $36.6$	& $26.2$	& $24.8$	\\
ThePile  & CB(FW-EDU)	& $52.0$	& $31.1$	& $54.5$	& $71.4$	& $45.6$	& $36.0$	& $26.2$	& $24.8$	\\
ThePile & \TBDF(FW-EDU)	& $53.4$	& $33.4$	& $56.7$	& $72.6$	& $45.6$	& $36.9$	& $26.5$	& $27.3$	\\
ThePile & \TBDF(General)	& $54.6$	& $31.1$	& $56.6$	& $74.1$	& $46.8$	& $38.2$	& $26.6$	& $25.9$	\\
\hline
C4 & Baseline	& $59.9$	& $28.7$	& $57.0$	& $76.5$	& $47.8$	& $37.2$	& $26.6$	& $23.9$	\\
C4  & CB(FW-EDU)	& $62.0$	& $31.2$	& $55.9$	& $76.9$	& $47.5$	& $39.4$	& $25.8$	& $23.1$	\\
C4 & \TBDF(FW-EDU)	& $61.6$	& $31.8$	& $58.2$	& $76.7$	& $48.2$	& $39.1$	& $27.2$	& $25.9$	\\
C4 & \TBDF(General)	& $62.6$	& $32.3$	& $57.9$	& $77.1$	& $48.5$	& $40.6$	& $26.3$	& $26.7$	\\
\hline
FineWeb  & Baseline	& $58.5$	& $30.9$	& $54.9$	& $75.2$	& $47.0$	& $38.2$	& $25.4$	& $23.6$	\\
FineWeb  & CB(FW-EDU)	& $57.8$	& $39.9$	& $57.3$	& $74.1$	& $45.8$	& $38.7$	& $24.5$	& $24.0$	\\
FineWeb  & \TBDF(FW-EDU)	& $60.3$	& $40.4$	& $56.8$	& $75.9$	& $47.5$	& $38.2$	& $27.5$	& $26.2$	\\
FineWeb  & \TBDF(General)	& $62.5$	& $38.1$	& $58.3$	& $77.5$	& $46.7$	& $40.1$	& $25.2$	& $26.9$	\\
\hline
\hline Model size & \multicolumn{9}{c|}{4B} \\ \hline
Dataset & Filtering  & HellaSwag	& ARC	& Winogrande	& PIQA	& SIQA	& CQA	& MMLU	& BookQA	\\ \hline \hline 
ThePile & Baseline	& $65.5$	& $38.2$	& $62.3$	& $76.7$	& $47.8$	& $41.6$	& $27.2$	& $28.2$	\\
ThePile  & CB(FW-EDU)	& $64.6$	& $37.8$	& $60.8$	& $76.6$	& $47.5$	& $41.1$	& $27.3$	& $28.9$	\\
ThePile & \TBDF(FW-EDU)	& $67.7$	& $42.0$	& $61.2$	& $77.8$	& $48.0$	& $44.5$	& $30.2$	& $29.0$	\\
ThePile & \TBDF(General)	& $66.7$	& $40.4$	& $62.8$	& $78.0$	& $48.6$	& $42.3$	& $28.5$	& $34.5$	\\
\hline
C4 & Baseline	& $72.8$	& $35.8$	& $64.3$	& $79.1$	& $48.5$	& $43.1$	& $26.8$	& $24.5$	\\
C4  & CB(FW-EDU)	& $73.4$	& $40.1$	& $64.7$	& $79.8$	& $48.3$	& $43.1$	& $26.5$	& $23.4$	\\
C4 & \TBDF(FW-EDU)	& $72.6$	& $39.2$	& $64.7$	& $79.7$	& $48.4$	& $42.4$	& $25.6$	& $25.3$	\\
C4 & \TBDF(General)	& $73.7$	& $40.0$	& $66.0$	& $79.7$	& $49.7$	& $43.4$	& $25.8$	& $27.3$	\\
\hline
FineWeb  & Baseline	& $71.1$	& $37.9$	& $63.1$	& $78.4$	& $48.8$	& $43.1$	& $25.4$	& $22.2$	\\
FineWeb  & CB(FW-EDU)	& $68.6$	& $48.5$	& $60.6$	& $78.2$	& $49.1$	& $42.7$	& $26.0$	& $27.5$	\\
FineWeb  & \TBDF(FW-EDU)	& $70.7$	& $47.3$	& $62.3$	& $79.6$	& $48.8$	& $42.3$	& $27.9$	& $27.6$	\\
FineWeb  & \TBDF(General)	& $72.5$	& $45.5$	& $63.5$	& $79.7$	& $49.5$	& $43.8$	& $25.7$	& $25.9$	\\
\hline
\end{tabular}
    \caption{Ablation results for each evaluation task and data filtering method. We used Gemma 3 models with 270M, 1B and 4B parameters~\cite{gemmateam2025gemma3technicalreport}.}
    \label{tab:ablation}
\end{table}


\end{document}